\documentclass[letterpaper, 10 pt, conference]{ieeeconf}  %

\IEEEoverridecommandlockouts                              %

\usepackage{times}
\usepackage{algorithm}
\usepackage{algpseudocode}
\usepackage{amsmath}
\usepackage{amssymb}
\usepackage{graphicx}
\usepackage{array}
\usepackage{multirow}
\usepackage[table]{xcolor}
\usepackage{cite}
\usepackage{subcaption}
\usepackage{mysymbol}

\newtheorem{theorem}{Theorem}[section]

\newtheorem{prop}[theorem]{Proposition}

\newtheorem{problem}{Problem}

\newtheorem{rem}[theorem]{Remark}

\algtext*{EndIf}
\algtext*{EndFor}
\algtext*{EndWhile}
\algtext*{EndFunction}
\algtext*{EndProcedure}

\usepackage{booktabs}

\title{\LARGE \bf
Conformalized Non-uniform Sampling Strategies \\ for Accelerated Sampling-based Motion Planning
}

\author{Shubham Natraj \and Bruno Sinopoli \and Yiannis Kantaros%
\thanks{The authors are with Department of Electrical and Systems Engineering, Washington University in St. Louis, MO, 63108, USA. Emails:
        {\tt\small \{n.shubham,bsinopoli,ioannisk\}@wustl.edu}%
}
}

\begin{document}
\bstctlcite{IEEEexample:BSTcontrol}

\maketitle
\thispagestyle{empty}
\pagestyle{empty}

\begin{abstract}
Sampling-based motion planners (SBMPs) are widely used to compute dynamically feasible robot paths. However, their reliance on uniform sampling often leads to poor efficiency and slow planning in complex environments. We introduce a novel non-uniform sampling strategy that integrates into existing SBMPs by biasing sampling toward `certified' regions. These regions are constructed by (i) generating an initial, possibly infeasible, path using any heuristic path predictor (e.g., A* or vision-language models) and (ii) applying conformal prediction to quantify the predictor's uncertainty. This process yields prediction sets around the initial-guess path that are guaranteed, with user-specified probability, to contain the optimal solution. To our knowledge, this is the first non-uniform sampling approach for SBMPs that provides such probabilistically correct guarantees on the sampling regions. Extensive evaluations demonstrate that our method consistently finds feasible paths faster and generalizes better to unseen environments than existing baselines.
\end{abstract}

\section{Introduction}\label{sec:intro}

Sampling-based motion planners (SBMPs) are powerful methods for designing dynamically feasible robot paths in high-dimensional configuration spaces, particularly for reach-avoid tasks \cite{kavraki1996probabilistic,lavalle1998rapidly,karaman2011sampling,janson2015fast,wang2025motion}. They incrementally construct graph structures—typically trees—by sampling and adding nodes, where each node represents a robot configuration and edges model control actions. Despite their probabilistic completeness and asymptotic optimality guarantees, SBMPs often suffer from poor sample efficiency and slow planning, especially in cluttered environments, due to their reliance on uniform sampling of robot configurations.

To address this limitation, we propose a non-uniform sampling strategy that biases sampling toward promising regions likely to contain optimal paths. A key challenge is designing \textit{certified sampling regions} that contain solutions with high probability—otherwise, the planner may waste time exploring invalid directions, further slowing the process. Our strategy addresses this via two key ideas. 
First, we generate a fast initial path from start to goal using any existing heuristic planner, such as A* \cite{hart1968formal}, Vision-Language Models (VLMs) \cite{song2025vl}, or motion planning networks \cite{qureshi2020motion}. This path need not be valid but it serves as the backbone of our sampling regions; for example, A* typically ignores system dynamics, while VLMs and planning networks may produce paths that collide with obstacles. Second, we employ conformal prediction (CP), a statistical framework for uncertainty quantification in black-box models \cite{ shafer2008tutorial, angelopoulos2023conformal}, to quantify uncertainty of these path predictors and construct prediction sets around this initial path. These sets guarantee, with user-specified probability, the existence of the optimal path within them. We then bias sampling toward these certified regions. To our knowledge, this is the first non-uniform sampling strategy with theoretical guarantees that the biased regions contain the optimal path with a desired probability. We integrate our CP-based strategy into $\text{RRT}^*$ and demonstrate through extensive evaluations that it significantly accelerates path computation and generalizes better to unseen environments than existing baselines \cite{ichter2018learning, lavalle2001randomized}. This performance gap widens as the complexity of the environment increases. \textcolor{black}{Finally, we note that although we focus on $\text{RRT}^*$ in this work, our approach is broadly applicable to other SBMPs 
\cite{kavraki1996probabilistic,lavalle1998rapidly,karaman2011sampling,janson2015fast,wang2025motion}
} %

\textbf{Related Works:} Several non-uniform sampling strategies have been proposed to improve planning efficiency.
\textbf{(i)} \emph{Heuristic approaches} have been developed to bias the sampling process toward promising regions by leveraging task-specific information (e.g., goal region), environment geometry, or path cost heuristics 
\cite{lavalle2001randomized,hsu1998finding,urmson2003approaches,hsu2003bridge,van2005using,yang2004adapting,qureshi2016potential,luo2021abstraction}. While these methods can offer significant performance gains, they typically require expert-designed heuristics or parameter tuning. Moreover, empirical results demonstrate that their effectiveness may degrade in environments that deviate from their expected operating conditions \cite{urmson2003approaches}. 
\textbf{(ii)} \textit{Informed sampling methods} have also been proposed to accelerate convergence to an optimal path by leveraging the current best path \cite{gammell2018informed}. However, these methods, unlike ours, require an initial \textit{feasible} solution before refining the sampling distribution. As such, they inherit the same limitations as standard SBMPs in cluttered environments, where finding an initial solution is already challenging. These methods are complementary to ours and could potentially be integrated with our approach once an initial path is found.
\textbf{(iii)} \emph{Learning-enabled sampling strategies} have also been proposed recently that leverage prior planning experiences to train neural networks that guide sampling in novel planning problems \cite{burns2005sampling,baldwin2010non,zucker2008adaptive,berenson2012robot,wang2020neural,molina2020learn,huang2024neural,qureshi2018deeply,ichter2018learning}. 
Although these approaches result in significant performance improvements, they often require extensive data collection and computationally expensive training procedures. Our approach avoids these issues as it can leverage any planner (e.g., A*) to generate an initial, and potentially infeasible, path and only requires a small calibration set for CP. Moreover, unlike many learning-based methods, we demonstrate stronger empirical generalization to novel environments. We emphasize again that, unlike all the works discussed in (i)–(iii), our proposed sampling strategy identifies \textit{certified regions}---regions that are guaranteed to contain the optimal solution with a user-specified probability—to guide the sampling process resulting in significant performance gains. \textbf{(iv)} Finally, CP has been applied to a variety of autonomy tasks for uncertainty quantification, including object tracking \cite{su2024collaborative}, perception \cite{dixit2024perceive}, trajectory prediction \cite{lindemann2023safe}, and language-instructed robot planning \cite{wang2024probabilistically}. To our knowledge, this work presents the first application of CP to \textcolor{black}{quantify uncertainty of heuristic path predictors (e.g., A*, VLM-based planners, and planning networks) serving as the basis for developing our non-uniform sampling strategies for kinodynamic SBMPs.}

\textbf{Contributions:} The main contributions of this paper are as follows.
\textit{First}, we introduce the first certified non-uniform sampling strategy, enabled by CP, that guides the sampling process toward regions guaranteed to contain the optimal path with a user-specified probability.
\textit{Second}, we develop the first CP framework to quantify the uncertainty of any heuristic path predictor; this uncertainty quantification outcome is used to design the proposed sampling strategy.
\textit{Third}, unlike prior learning-based methods that require training task-specific models, our framework is training-free and uses only a small calibration set to apply CP. Additionally, our approach is versatile and agnostic to the employed path predictor. %
\textit{Fourth}, we present extensive comparative experiments demonstrating that our method consistently outperforms existing baselines in terms of runtime to compute a feasible solution and generalization to unseen environments. %

\section{Problem Formulation}\label{sec:pr}

Consider a mobile robot with dynamics  
\begin{equation}\label{eq:rdynamics}
    \dot{\bbx}(t)=f(\bbx(t),\bbu(t)) 
\end{equation}
where $\bbx(t)\in\ccalX\subset\mathbb{R}^n$, $n\geq 2$, denotes the robot state (e.g., position and heading) and $\bbu(t)\in\ccalU$ stands for the control input at time $t$. Also, let $\ccalX_{\text{obs}} \subset \ccalX$ denote the obstacle/unsafe space and $\ccalX_{\text{free}}=\ccalX\setminus\ccalX_{\text{obs}}$ denote the obstacle-free space. Consider also an initial state $\bbx_{\text{init}}\in\ccalX$ and a goal region $\ccalX_{\text{goal}}\subset\ccalX_{\text{free}}$. For brevity, we denote a planning problem as a tuple $\mathcal{M} = (\ccalX, \ccalX_{\text{free}},\bbx_{\text{init}},\ccalX_{\text{goal}})$. We assume that planning problems $\ccalM$ are sampled from a distribution $\mathcal{D}$. \textcolor{black}{This distribution is unknown but we assume that we can draw i.i.d. samples from it.} %
The planning problem that this paper addresses can be summarized as follows \cite{karaman2011sampling}:

\begin{problem}\label{pr:pr1}
Given a planning problem $\mathcal{M} \sim \mathcal D$, compute a control function $\bbu:[0,1]\rightarrow \ccalU$ such that the corresponding unique trajectory $\rho:[0,1]\rightarrow \ccalX$ determined by solving \eqref{eq:rdynamics} satisfies (i) $\rho(0)=\bbx_{\text{init}}$; (ii) $\rho(1)\in\ccalX_{\text{goal}}$; (iii) $\rho(t)\in\ccalX_{\text{free}}$, for all $t\in [0,1]$; and minimizes a motion cost function $J(\bbx)=\int_0^1 g(\bbx(t))$ (e.g., length of trajectory).
\end{problem}

As discussed in Section~\ref{sec:intro}, several SBMP approaches have been proposed to address Problem~\ref{pr:pr1}. These methods construct graphs over the state space $\ccalX$ by sampling and adding nodes that represent robot states $\bbx$, with edges modeling the control actions required to reach them. As the number of samples increases, the probability of finding a solution approaches one. However, the efficiency of SBMPs heavily depends on the chosen sampling strategy. In this work, we focus on designing non-uniform sampling strategies to accelerate the computation of feasible paths using SBMPs.

\section{Methodology}
\textcolor{black}{In this section, we develop a non-uniform sampling strategy to accelerate SBMPs  addressing Problem \ref{pr:pr1}. Specifically, in Section \ref{sec:meth-cp}, we leverage Conformal Prediction (CP) to quantify the uncertainty of path predictors that can quickly but approximately solve Problem \ref{pr:pr1}. This uncertainty quantification yields prediction sets of trajectories that include the optimal solution with a user-defined probability. In Section \ref{sec:proposed}, we use these prediction sets to develop a novel non-uniform sampling strategy for SBMPs that guides sampling toward regions likely to contain an optimal trajectory. This CP-driven sampling strategy is integrated within RRT* \cite{karaman2011sampling}, resulting in our proposed approach, CP-RRT*.}

\subsection{Uncertainty Quantification of Path Predictors}
\label{sec:meth-cp}

\begin{algorithm}[t]
\caption{CP Calibration}
\label{alg:cal-path}
\textbf{Input:} $N_{\text{cal}}, \alpha$\\
\textbf{Output:}  $\hat q$
\begin{algorithmic}[1]
\State $\ccalS_{\mathrm{cal}} =\emptyset$\label{algcpline:init}
\For{$i = 1 \to N_{\text{cal}}$}

    \State sample $\mathcal{M}_i \sim \mathcal{D}$ \label{algcpline:sample-prob}

    \State $\rho_{i} \;\gets\; \textsc{GetOptimalSolution}\bigl(\mathcal{M}_i)$ \label{algcpline:get-optimal-soln}
    
    \State $\textsc{Append}(\ccalS_{\mathrm{cal}},\,(\mathcal{M}_i , \rho_{i}))$ \label{algcpline:append-cal-pair}
    
\EndFor
\State $\mathcal{S} \;=\; \Bigl\{s(\sigma(\ccalM_i) , \rho_i) \;\Big|\; \forall i \in [1,\dots,N_{\text{cal}}]\Bigr\}$ \label{algcpline:compute-scores}
\State $r \gets \lceil(1-\alpha)(N_{\text{cal}}+1)\rceil$ \label{algcpline:compute-r}
\State $\hat q \gets (\textsc{Sort}(\mathcal{S}))_r$ \label{algcpline:compute-qhat}
\State \Return $\hat q$
\end{algorithmic}
\end{algorithm}

\textcolor{black}{Consider any path predictor, denoted by $\sigma: \ccalM\rightarrow\ccalX^{L_p}$, which given a planning problem $\ccalM$ returns a path $\bbp_{1:L_p} = \bbp(1), \bbp(2), \dots, \bbp(L_p)$ where each state $\bbp(k) \in \ccalX$ for $k\in\{1, \dots, L_p\}$, $\bbp(1) = \bbx_{\text{init}}$, $\bbp(L_p) \in \ccalX_{\text{goal}}$, and $L_p\in\mathbb{N}_{+}$. This path predictor can be implemented using existing approaches such as A*, VLM-based path planners \cite{song2025vl}, or motion planning networks \cite{qureshi2020motion}.
Due to imperfections of the employed predictors, the generated path may be infeasible, in the sense that (i) some states $\bbp(k)$ may lie outside the obstacle-free space $\ccalX_{\text{free}}$, or (ii) there may be no control input $\bbu$ that drives the system from $\bbp(k)$ to $\bbp(k+1)$ according to \eqref{eq:rdynamics}. The path may also be approximate, meaning that $\bbp(k)$ may belong to a lower-dimensional space $\hat{\ccalX}\subset\ccalX$; e.g., if $\bbx(t) \in \mathbb{R}^3$ encodes both position and heading, $\bbp(k)\in \mathbb{R}^2$ may include only the position. Note that the index $k$ is used to point to the $k$-th state in $\bbp_{1:L_p}$ and is different from $t$.} 

\textcolor{black}{
Next, we quantify the uncertainty of the predictor $\sigma$ to account for such imperfections. Specifically, given a predicted path $\bbp_{1:L_p}$, we design a \textit{prediction set} $\ccalC(\bbp_{1:L_p})$ of trajectories such that the optimal one, as per Problem \ref{pr:pr1}, lies entirely in $\ccalC(\bbp_{1:L_p})$ with user-specified probability $1-\alpha \in [0,1)$, i.e.,
\begin{equation}
\label{eqn:guarantee}
    \mathbb{P}\Bigl( \rho \in  \ccalC(\bbp_{1:L_p})\Bigr) \;\ge\; 1-\alpha \;\; %
\end{equation}
where $\rho : [0,1] \rightarrow \ccalX$ is the optimal trajectory. This set will be used in Section \ref{sec:proposed} to accelerate SBMPs.}

\textcolor{black}{Our approach leverages conformal CP to compute prediction sets and is agnostic to the specific path predictor $\sigma$. Applying CP requires:
(i) a calibration dataset $\ccalS_{\text{cal}}={\{\ccalM_i,\rho_{i}\}}_{i=1}^{N_{\text{cal}}}$ consisting of planning problems $\ccalM_i \sim \ccalD$ paired with their optimal solutions $\rho_i$, and
(ii) a \textit{non-conformity score (NCS)} that quantifies the error in the predictor’s outputs. Given (i)–(ii) and a new planning problem $\ccalM \sim \ccalD$ with unknown feasible/optimal solutions, CP produces a prediction set satisfying \eqref{eqn:guarantee}. In what follows, we describe this process in detail; see Alg. \ref{alg:cal-path}.}

\textcolor{black}{
\textbf{Calibration Dataset:} Consider any new planning problem $\ccalM\sim\ccalD$ with unknown feasible/optimal solutions and let $\bbp_{1:L_p}$ be the solution predicted by $\sigma$. To construct $\ccalC(\bbp_{1:L_p})$ we first build a calibration dataset, $\ccalS_{\text{cal}}=\{\ccalM_i,\rho_i\}_{i=1}^{N_{\text{cal}}}$, by sampling $N_{\text{cal}}$ planning problems $\ccalM_i$ from $\ccalD$. Then using any existing optimal kinodynamic planner, we compute their corresponding optimal trajectory, $\rho_i : [0,1] \rightarrow \ccalX$. %
}

\textbf{Non-Conformity Score (NCS):} %
We define the NCS as a function $s : \ccalX^{L_p} \times \mathcal{F} \to \mathbb{R}$ that measures the error/distance between a path $\bbp_{1:L_p}$ predicted by $\sigma$ and an optimal trajectory $\rho \in \mathcal{F}$, where $\mathcal{F}$ denotes the space of continuous functions (i.e., trajectories) $\rho:[0,1]\to\mathcal{X}$. Specifically, given $\bbp_{1:L_p}$ and $\rho$, we define the NCS as follows:

\begin{equation}
\label{eqn:score}
    s(\bbp_{1:L_p},\rho) \;=\; 
\max_{k\in\{1,\dots,L_p\}}\;\max_{\substack{\rho(n) \in \ccalV(k), n\in[0,1]}} 
\Bigl\|\bbp_{1:L_p}(k) - \rho(n)\Bigr\|,
\end{equation}
\textcolor{black}{where $\ccalV(k)$ is the Voronoi cell associated with $\bbp(k)$ and generated by considering all states in the path  $\bbp_{1:L_p}$, i.e.,}
\begin{align}
\label{eqn:voronoi}
    \ccalV(k) = &\bigl\{\bbq \in \ccalX \;\big|\; 
    \|\bbq - \bbp_{1:L_p}(k)\| \le \|\bbq - \bbp_{1:L_p}(\ell)\|,\\& ~~~~~~\forall \ell \in \{1,\dots,L_p\}\setminus\{k\}\bigr\}.\nonumber
\end{align}
\textcolor{black}{The inner maximization in \eqref{eqn:score} computes the maximum Euclidean distance between a fixed state $\bbp(k)$ and all states  along the trajectory $\rho$ that lie within the Voronoi cell $\ccalV(k)$ of $\bbp(k)$. Then the outer maximization computes the maximum of such distances across all states $\bbp(k)$.} Observe that the NCS is designed so that every state along the optimal trajectory is assigned to a Voronoi region of some predicted state. In other words, for every $n\in[0,1]$, there exists a unique $k\in\{1,\dots,L_p\}$, such that $\rho(n)\in\ccalV(k)$. This construction ensures no states in the optimal trajectory are omitted in the score computation. Computation of this score function is shown in Figure \ref{fig:score}; see also Rem. \ref{rem:calibNCS}.

\begin{figure}[t]
\begin{subfigure}[b]{0.49\columnwidth}
      \centering
        \includegraphics[width=\linewidth]{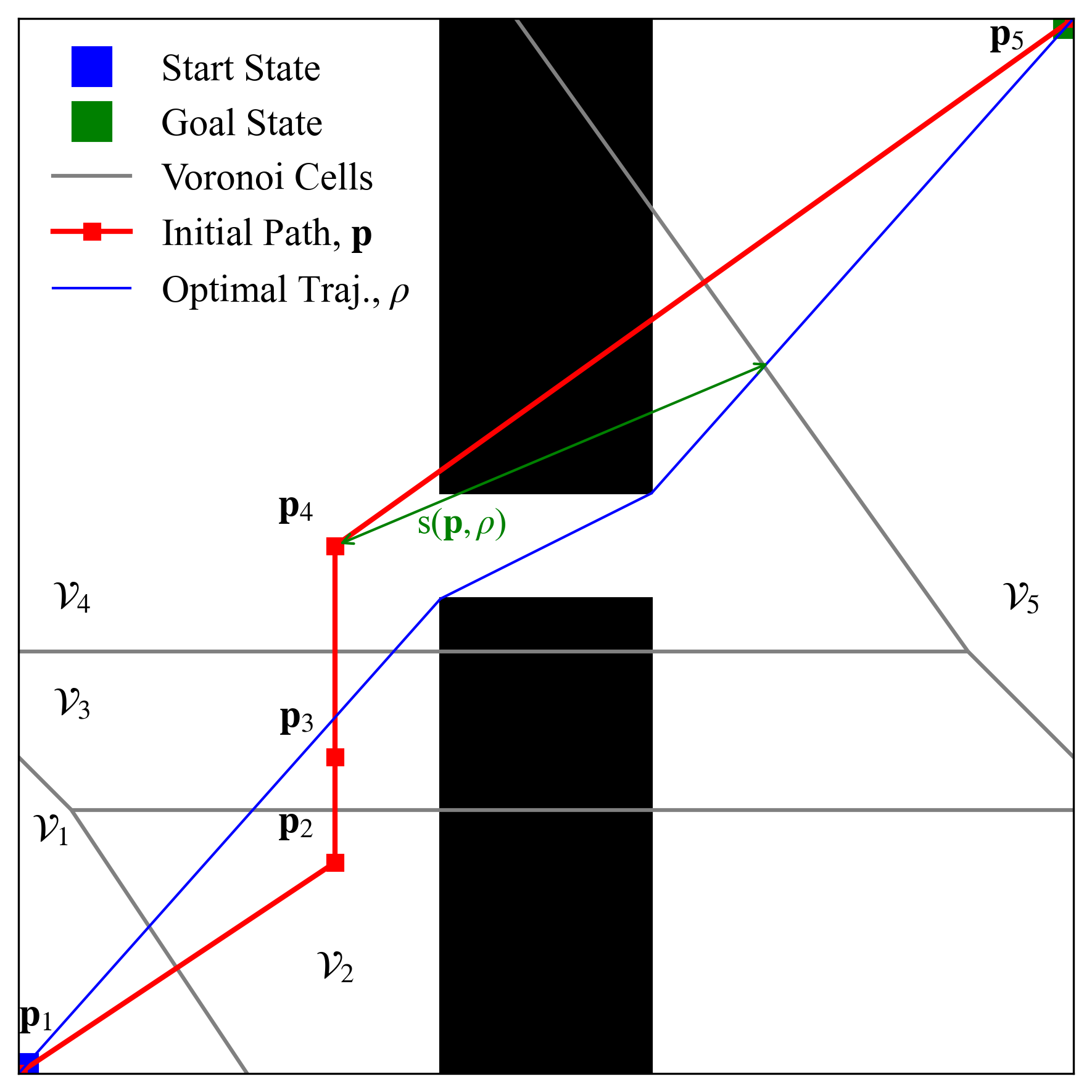}
        \caption{Nonconformity score.}
        \label{fig:score}
  \end{subfigure}
  \begin{subfigure}[b]{0.49\columnwidth}
      \centering
        \includegraphics[width=\linewidth]{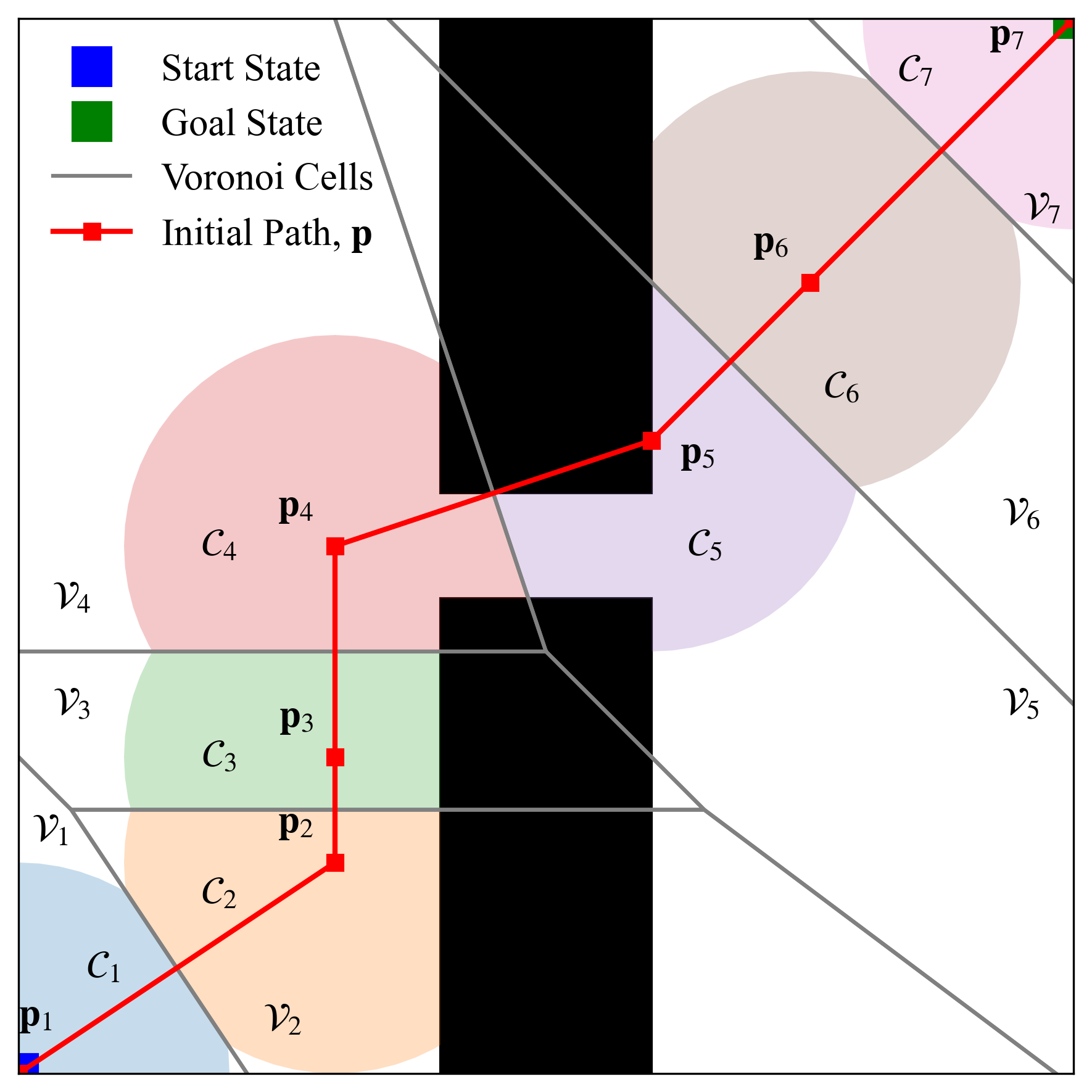}
        \caption{Prediction regions.}
        \label{fig:pred_regions}
  \end{subfigure}
\caption{
Illustration of the NCS and the corresponding prediction regions. 
In (a), the NCS is computed as the maximum distance between each predicted state $\bbp(k)$ and the portion of the optimal trajectory $\rho$ (blue) within its Voronoi cell $\ccalV(k)$. 
In (b), the resulting point-wise prediction regions $\ccalC_k$ (shaded) are shown as the intersections of balls of radius $\hat q$ centered at $\bbp(k)$ with their respective Voronoi regions. 
The start and goal states are indicated in blue and green, respectively, and obstacles are shown in black.
}%
\end{figure}

\textbf{Calibration Step:} Next, for each calibration problem $\ccalM_i$, we apply the path predictor to compute the corresponding predicted path $\sigma(\ccalM_i)$ and  evaluate the corresponding NCS $s(\sigma(\ccalM_i) , \rho_{i})$ in \eqref{eqn:score} [line \ref{algcpline:compute-scores}, Alg. \ref{alg:cal-path}].  Then, we find the $(1-\alpha)$-th empirical quantile of these NCSs denoted by $\hat{q}$ [lines \ref{algcpline:compute-r},  \ref{algcpline:compute-qhat}, Alg. \ref{alg:cal-path}], which is used to design sets satisfying \eqref{eqn:guarantee}.

\textbf{Prediction Sets:} Given the planning problem $\mathcal{M}$ and its predicted solution $\bbp_{1:L_p}=\sigma \bigl(\mathcal{M})$, we construct the following \emph{prediction set} of trajectories:
\begin{equation}
\label{eqn:conf_pred_region}
    \mathcal{C}\bigl(\bbp_{1:L_p}\bigr) 
    \;=\;
    \Bigl\{
   \rho \in\ccalF\Big|\; s\bigl(\bbp_{1:L_p},\rho\bigr) \;\le\; \hat{q}
    \Bigr\},
\end{equation}
\textcolor{black}{By construction of the NCS, we have that if the optimal trajectory satisfies $s(\bbp_{1:L_p},\rho) \leq \hat q$, then we have that the entire trajectory $\rho$ lies within $\ccalC(\bbp_{1:L_p})$.} Also, by standard CP argument, we have that the prediction set in \eqref{eqn:conf_pred_region} satisfies \eqref{eqn:guarantee}, implying that the optimal trajectory $\rho$ for $\ccalM$ will lie entirely in $\mathcal{C}\bigl(\bbp_{1:L_p}\bigr)$ with probability at least $1-\alpha$ \cite{angelopoulos2023conformal}.  

\textbf{Decomposition:} Since $\mathcal{C}\bigl(\bbp_{1:L_p}\bigr)$ is a set of trajectories, it is impractical to characterize exactly (and, therefore, use it to guide sampling in Section \ref{sec:proposed}). Hence, we decompose it into point-wise prediction regions. For each state $\bbp(k)$ in the predicted path, we define the associated point-wise prediction set $\ccalC_k$ as the intersection of the ball of radius $\hat q$ centered at $\bbp(k)$ and the Voronoi cell $\ccalV(k)$ defined in \eqref{eqn:voronoi}, i.e.,
\begin{equation}
\label{eqn:decomposed_pred_reg}
    \ccalC_k=\{\bbq\in\ccalV(k)~|~\|\bbq - \bbp_{1:L_p}(k)\|\leq \hat{q} \}.
\end{equation}
Figure \ref{fig:pred_regions} illustrates these  point-wise prediction sets. Then, in Proposition \ref{prop:decomp}, we show that we can re-write the prediction set in \eqref{eqn:conf_pred_region} as follows
\begin{equation}\label{eqn:decompSet}
    \mathcal{C}(\bbp_{1:L_p}) 
    \;=\;
    \Bigl\{z\in\ccalF \;\Big|\; 
    z(n) \in \bigcup_{k=1}^{L_p} \ccalC_k, \;\; \forall n \in [0,1] \Bigr\}.
\end{equation}

\begin{prop}[Set Decomposition]\label{prop:decomp}
The prediction set in \eqref{eqn:conf_pred_region} can be re-written equivalently as in \eqref{eqn:decompSet}.
\end{prop}

\begin{proof}
To prove this result, it suffices to show that any $z$ belonging to the set in \eqref{eqn:conf_pred_region} also belongs to the set in \eqref{eqn:decompSet}, and vice versa.
By construction of the prediction set in ~\eqref{eqn:conf_pred_region} and the NCS in \eqref{eqn:score}, we have that a trajectory $z:[0,1]\to\ccalX$ belongs to $\mathcal{C}(\bbp_{1:L_p})$ if and only if 
\begin{equation}
    \max_{k\in\{1,\dots,L_p\}}\;\max_{\substack{z(n) \in \ccalV(k), n\in[0,1]}} 
    \bigl\|\bbp_{1:L_p}(k) - z(n)\bigr\| \;\le\; \hat q.
\end{equation}
This condition means that if $z$ belongs to the set in \eqref{eqn:conf_pred_region}, then for each $n \in [0,1]$, the state $z(n)$ lies within distance $\hat q$ of some predicted state $\bbp(k)$ whose Voronoi cell contains $z(n)$. In other words, $z(n) \in \bigcup_{k=1}^{L_p} \ccalC_k$ for all $n$, and therefore $z$ belongs to the set in \eqref{eqn:decompSet}. Using the same logic, we can show that if $z(n)\in\bigcup_{k=1}^{L_p}\ccalC_k$ for all $n\in[0,1]$, then $z$ belongs to the set in \eqref{eqn:conf_pred_region} completing the proof.
\end{proof}

\begin{rem} [Calibration Dataset \& NCS]\label{rem:calibNCS}
In practice, constructing calibration datasets using continuous-time optimal trajectories and evaluating the NCS in \eqref{eqn:score} directly on them can be computationally challenging. A tractable alternative is to use approximately optimal solutions (e.g.,, by running SBMPs for a large number of iterations to obtain a discretized trajectory) and then evaluate the NCS on it. 
\end{rem}

\subsection{CP-RRT*: RRT* with CP-driven Sampling Strategy}\label{sec:proposed}

\begin{algorithm}[t]
\caption{CP-RRT*}
\label{alg:cp-RRT*}
\textbf{Input:} $\mathcal{M} = (\ccalX, \ccalX_{\text{free}},\bbx_{\text{init}},\ccalX_{\text{goal}}),\,\hat{q}, \,p_{bias}$ \\
\textbf{Output:} $\mathcal{T}$ 
\begin{algorithmic}[1]
\State $\ccalV \leftarrow \{\bbx_{\text{init}}\}$,\quad $\ccalE \leftarrow \varnothing$,\quad $\mathcal{T} \leftarrow (\ccalV,\ccalE)$ \label{rrtline:init-tree}
\State $\bbp_{1:L_p} \gets \sigma(\ccalM)$ \label{rrtline:get-initial-path}
\State Compute prediction sets $\ccalC_k$ as in \eqref{eqn:decomposed_pred_reg}, $\forall k \in \{1,\dots,L_p\}$\label{rrtline:construct-pred-sets}
\For{$j = 1 \to N$} \label{rrtline:for-loop-start}
    \State $\bbx_{\mathrm{rand}} \leftarrow \textsc{SAMPLE}(\{\ccalC_k\}_{k=1}^{L_p},\; p_{\mathrm{bias}},\; \ccalX_{\mathrm{free}})$ \label{rrtline:sample}
    \State $\bbx_{\mathrm{nearest}} \leftarrow \textsc{Nearest}\bigl(\mathcal{T},\, \bbx_{\mathrm{rand}}\bigr)$ \label{rrtline:nearest}
    \State $\bbx_{\mathrm{new}} \leftarrow \textsc{Steer}\bigl(\bbx_{\mathrm{nearest}},\, \bbx_{\mathrm{rand}}\bigr)$ \label{rrtline:steer}
    \If{$\textsc{ObstacleFree}\bigl(\bbx_{\mathrm{nearest}},\, \bbx_{\mathrm{new}}\bigr)$} \label{rrtline:obstaclefree-check}
        \State $\mathcal{X}_{\mathrm{near}} \leftarrow \textsc{Near}\bigl(\mathcal{T},\, \bbx_{\mathrm{new}},\, r_n\bigr)$ \label{rrtline:near-nodes}
        \State $\bbx_{\mathrm{parent}} \leftarrow \textsc{ChooseParent}\bigl(\mathcal{T},\, \mathcal{X}_{\mathrm{near}},\, \bbx_{\mathrm{new}}\bigr)$ \label{rrtline:choose-parent}
        \State $\mathcal{T} \leftarrow \textsc{Extend}\bigl(\mathcal{T},\, \bbx_{\mathrm{new}},\, \bbx_{\mathrm{parent}}\bigr)$ \label{rrtline:extend-tree}
        \State $\mathcal{T} \leftarrow \textsc{Rewire}\bigl(\mathcal{T},\, \bbx_{\mathrm{new}},\, \mathcal{X}_{\mathrm{near}}\bigr)$ \label{rrtline:rewire}
    \EndIf
\EndFor \label{rrtline:for-loop-end}
\State \Return $\mathcal{T}$ \label{rrtline:return-solution}
\end{algorithmic}
\end{algorithm}

In this section, we provide an overview of CP-RRT*; see Alg.~\ref{alg:cp-RRT*}. CP-RRT* augments RRT* with a non-uniform sampling strategy, driven by the prediction sets in \eqref{eqn:decomposed_pred_reg}; the remaining core RRT* subroutines %
remain unchanged \cite{karaman2011sampling}. As in RRT*, our algorithm incrementally builds a tree $\ccalT = \{\ccalV, \ccalE\}$, where $\ccalV \subseteq \ccalX$ and $\ccalE \subseteq \ccalV \times \ccalV$ are the sets of nodes and edges, respectively [line  \ref{rrtline:init-tree}, Alg. \ref{alg:cp-RRT*}]. These sets are initialized as $\ccalV = \{\bbx_{\text{init}}\}$ and $\ccalE = \emptyset$, and grown incrementally by adding new states to $\ccalV$ and corresponding edges to $\ccalE$, for a user-specified number of iterations, $N$. %

\begin{algorithm}[t]
\caption{SAMPLE}
\label{alg:sample}
\textbf{Input:} $\{\ccalC_k\}_{k=1}^{L_p},\; p_{\mathrm{bias}},\; \ccalX_{\mathrm{free}}$ \\
\textbf{Output:} $\bbx_{\mathrm{rand}}$
\begin{algorithmic}[1]
\State Compute $\ccalA_k = \ccalX_{\text{free}} \cap \ccalC_k,~\forall k \in \{1,\dots,L_p\}$ \label{a2l:create_a}
\State Select $k \in \{1,\dots,L_p\}$ \label{a2l:select_k}
\State Draw $\bbx_{\mathrm{rand}}|k  \sim (1-p_{\mathrm{bias}})\,\mathbb{U}_{\ccalX_{\text{free}}} + p_{\mathrm{bias}}\,\mathbb{U}_{\ccalA_k}.$ \label{a2l:sample_dist}
\State \Return $\bbx_{\mathrm{rand}}$
\end{algorithmic}
\end{algorithm}
\textbf{Computing an Initial Path:} In what follows we provide a more detailed description of Alg. \ref{alg:cp-RRT*}.
First, using the path predictor $\sigma$,  we compute a path $\bbp_{1:L_p} = \bbp(1), \bbp(2), \dots, \bbp(L_p)$, where $\bbp(k) \in \ccalX$ for $k\in\{1, \dots, L_p\}$, $\bbp(1) = \bbx_{\text{init}}$ and $\bbp(L_p) \in \ccalX_{\text{goal}}$ [line  \ref{rrtline:get-initial-path}, Alg. \ref{alg:cp-RRT*}]. 
As discussed in Section \ref{sec:meth-cp}, this path may be infeasible and approximate. Given this predicted path, we compute its associated prediction sets $\ccalC_k$ defined in \eqref{eqn:decomposed_pred_reg} [line  \ref{rrtline:construct-pred-sets}, Alg. \ref{alg:cp-RRT*}]. %

\textbf{Tree Construction:} At every iteration $n$, we sample $\bbx_{\text{rand}}\in\ccalX$ from a distribution that leverages the prediction sets in \eqref{eqn:decomposed_pred_reg}. The design of this distribution, which differentiates our algorithm from RRT* [lines \ref{rrtline:get-initial-path}, \ref{rrtline:construct-pred-sets}, \ref{rrtline:sample}, Alg. \ref{alg:cp-RRT*}], is discussed later.
Given a sample $\bbx_{\text{rand}}$, the tree is updated as in the RRT* algorithm. Specifically, first the \textsc{Nearest} operation identifies the vertex $\bbx_{\mathrm{nearest}} \in \ccalV$ closest to $\bbx_{\mathrm{rand}}$ under a given cost metric [line \ref{rrtline:nearest}, Alg. \ref{alg:cp-RRT*}]. Given $\bbx_{\mathrm{nearest}}$, the \textsc{Steer} function generates a control input $\bbu_{\text{new}}$ moving the robot from $\bbx_{\mathrm{nearest}}$ to a new state $\bbx_{\mathrm{new}}$ that is closer to $\bbx_{\mathrm{rand}}$ than $\bbx_{\mathrm{nearest}}$ is, while taking into account the robot dynamics~\eqref{eq:rdynamics} [line \ref{rrtline:steer}, Alg. \ref{alg:cp-RRT*}]. Then, the \textsc{ObstacleFree} function verifies that the steered trajectory from $\bbx_{\mathrm{nearest}}$ to  $\bbx_{\mathrm{new}}$ does not intersect any obstacles [line \ref{rrtline:obstaclefree-check}, Alg. \ref{alg:cp-RRT*}]. If this trajectory is collision-free, a set of neighbors $\mathcal{X}_{\mathrm{near}}$ is computed by \textsc{Near} [line  \ref{rrtline:near-nodes}, Alg. \ref{alg:cp-RRT*}], among which a parent is selected through \textsc{ChooseParent} that minimizes the path cost from the root $\bbx_\mathrm{init}$ to $\bbx_{\mathrm{new}}$ [line  \ref{rrtline:choose-parent}, Alg. \ref{alg:cp-RRT*}]. The \textsc{Extend} step adds $\bbx_{\mathrm{new}}$ with parent $\bbx_{\mathrm{parent}}$ to the tree $\mathcal{T}$ [line  \ref{rrtline:extend-tree}, Alg. \ref{alg:cp-RRT*}], and the \textsc{Rewire} step updates the parents of neighbors if the cost of reaching them through $\bbx_{\mathrm{new}}$ gets lower [line  \ref{rrtline:rewire}, Alg. \ref{alg:cp-RRT*}]. The algorithm runs for $N$ iterations. If a tree node lies within $\ccalX_{\mathrm{goal}}$, then backtracking from it to the root gives a feasible solution. Among all feasible solutions, the one with lowest cost is returned. %

\textbf{Non-uniform Sampling Strategy:} To accelerate the computation of a trajectory solving Problem \ref{pr:pr1}, we design a non-uniform sampling strategy for drawing samples $\bbx_{\text{rand}}$; see Alg. \ref{alg:sample}. Our strategy  guides sampling toward regions that contain the optimal solution with high probability by leveraging the prediction sets $\ccalC_k$, $k \in \{1,\dots,L_p\}$, defined in \eqref{eqn:decomposed_pred_reg}.
Specifically, first, we compute the sets $\ccalA_k = \ccalX_{\text{free}} \cap \ccalC_k$, for all $k \in \{1,\dots,L_p\}$ [line \ref{a2l:create_a}, Alg. \ref{alg:sample}]; see Fig. \ref{fig:pred_regions}. %
Second, given the sets $\ccalA_k$, we select an index $k \in \{1,\dots,L_p\}$ [line \ref{a2l:select_k}, Alg. \ref{alg:sample}]. This selection can be made arbitrarily—for example, by choosing $k$ uniformly at random, or by cycling deterministically through all possible values as the number $n$ of iterations increases. Third, conditioned on the selected index $k$, we draw a sample from the following distribution:
\begin{equation}
    \bbx_{\mathrm{rand}} \mid K\!=\!k \sim (1-p_{\mathrm{bias}})\,\mathbb{U}_{\ccalX_{\text{free}}} + p_{\mathrm{bias}}\,\mathbb{U}_{\ccalA_k},
\end{equation}
\noindent where $\mathbb{U}_{\ccalA_k}$ and $\mathbb{U}_{\ccalX_{\text{free}}}$ denote the uniform distributions over $\ccalA_k$ and $\ccalX_{\text{free}}$, respectively, and $p_{\mathrm{bias}}\in[0,1)$ is a user-defined parameter [line \ref{a2l:sample_dist}, Alg. \ref{alg:sample}]. Increasing $p_{\mathrm{bias}}$ strengthens the bias toward sampling states in $\ccalA_k$, while setting $p_{\mathrm{bias}} = 0$ recovers the uniform distribution over $\ccalX_{\text{free}}$.

\begin{figure}[t]
  \centering
  \begin{subfigure}[b]{0.49\columnwidth}\centering
    \includegraphics[width=\linewidth]{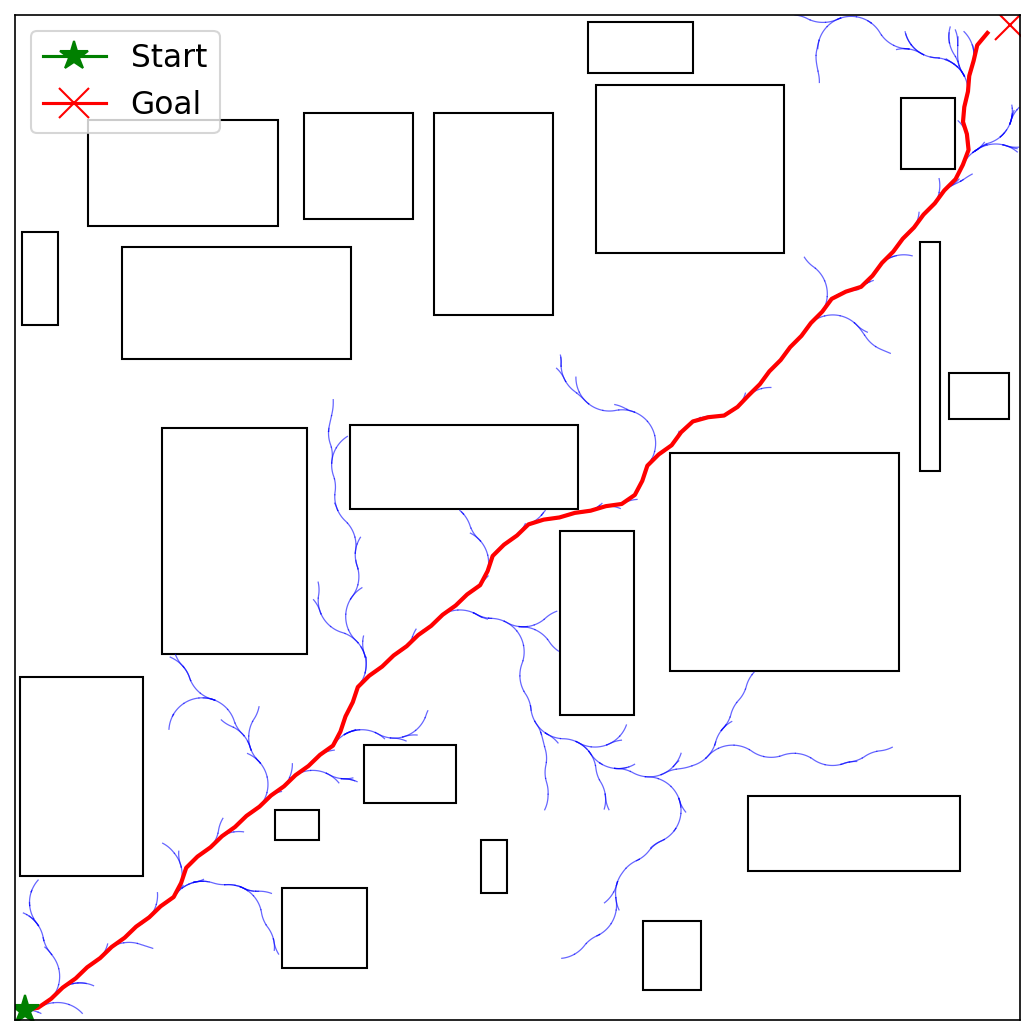}
    \caption{30\% density}
  \end{subfigure}
  \begin{subfigure}[b]{0.49\columnwidth}\centering
    \includegraphics[width=\linewidth]{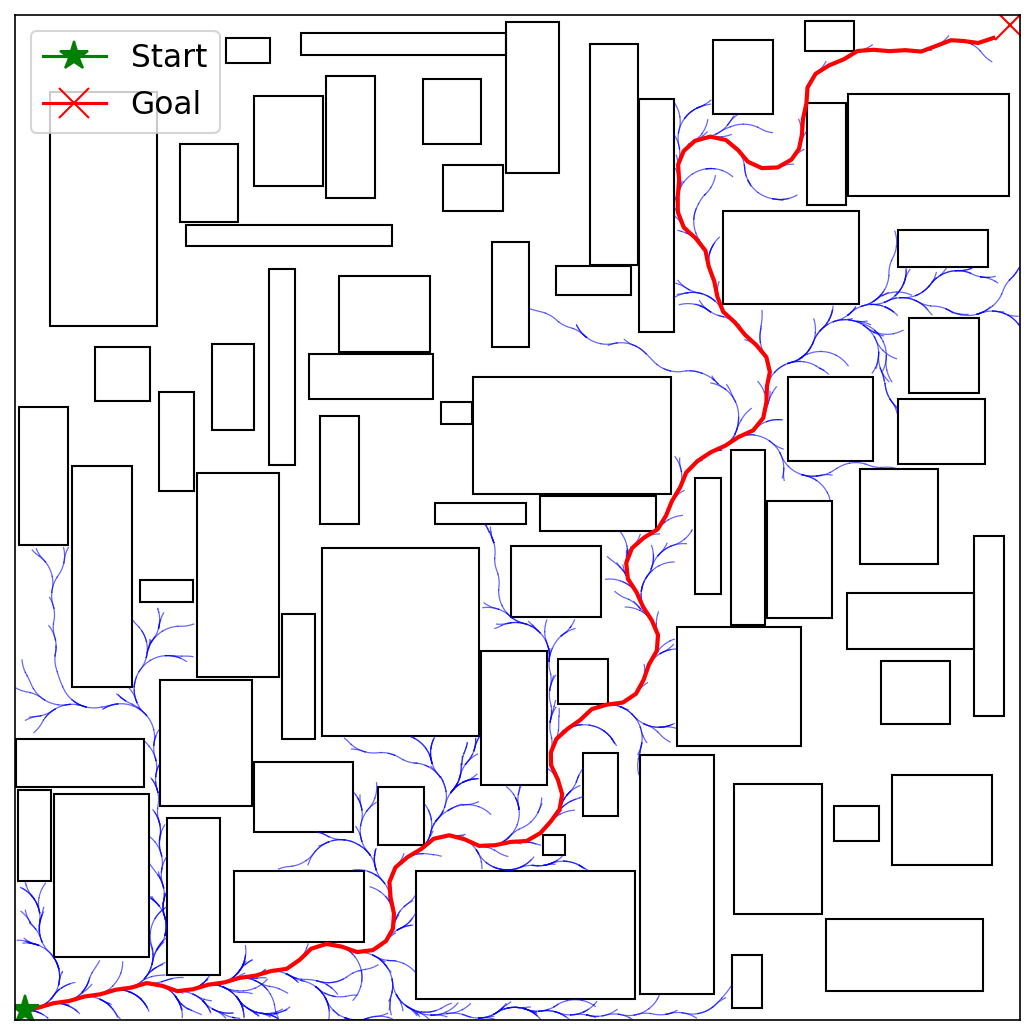}
    \caption{50\% density}
  \end{subfigure}
  \caption{Representative environments for 30\% and 50\% obstacle density classes with overlaid CP-RRT* trees.}%
  \label{fig:env_panels}
\end{figure}

\section{Experiments}
\label{sec:experiments}

This section evaluates CP-RRT* against heuristic and learning-enabled baselines. We first detail the experimental setup and the considered baselines (Section \ref{sec:exp-setup}). We then present comparative runtime experiments for both in-distribution (Section \ref{sec:exp-time}) and out-of-distribution settings (Section \ref{sec:exp-generalization}). We conclude with an evaluation of the effect of key planning parameters (Section \ref{sec:exp-ablation}) and demonstrate using a VLM as an alternate path predictor (Section \ref{sec:exp-predictors}).

\subsection{Setting Up CP-RRT* and Baselines}
\label{sec:exp-setup}

\textbf{System Dynamics:} We evaluate our approach using three robot models:
(i) A holonomic robot, with dynamics $\dot{\mathbf{x}}(t) = \mathbf{u}(t)$, where the state $\mathbf{x} = [x, y]^\top \in \mathbb{R}^2$ represents the robot’s planar position, and the control input $\mathbf{u} = [u_x, u_y]^\top \in \mathbb{R}^2$ corresponds to linear velocity.
(ii) A curvature-constrained Dubins vehicle \cite{dubins1957curves} with unit speed, where the state $\mathbf{x} = [x, y, \theta]^\top \in \mathbb{R}^3$ encodes planar position and heading, and the control $u \in \mathbb{R}$ denotes curvature, constrained by $|u| \le \kappa_{\max}$. The dynamics are given by $\dot{\mathbf{x}} = [\cos\theta, \sin\theta, u]^\top$. The steering function is implemented  using the Dubins Curves library \cite{DubinsCurves} which uses well-established results from \cite{dubins1957curves}. (iii) A 5-D kinematic car model, with state $\mathbf{x} = [x, y, \theta, v, \kappa]^\top \in \mathbb{R}^5$ capturing planar position $(x,y)$, heading $\theta$, forward speed $v$, and curvature $\kappa$. The control input $\mathbf{u} = [u_v, u_\kappa]^\top\in \mathbb{R}^2$ represents acceleration and curvature rate, and the dynamics are $\dot{\mathbf{x}} = [v\cos\theta, v\sin\theta, v\kappa, u_v, u_\kappa]^\top$. Local steering is achieved using an LQR controller following \cite{perez2012lqr}.

\textbf{Distribution of Planning Problems:} To assess how the performance of our algorithm is affected by the size of the obstacle space, we consider five distributions over planning problems $\ccalM$, denoted by $\ccalD_i$ that differ only in the obstacle density of the environment they generate.
Specifically, each $\ccalD_i$ generates 2D environments of size $100\text{m} \times 100\text{m}$, where $i\%$ of the area is occupied by obstacles. We consider obstacle densities $i \in \{10, 20, 30, 40, 50\}$. 
For all distributions, the initial state is fixed to the zero vector, and the goal region includes all states whose planar position $(x, y)$ lies within $3\text{m}$ of $(100, 100)$, the top right corner of the environment. %

\textbf{Path Predictor:} To evaluate how the choice of path predictor affects the performance of CP-RRT*, we consider two implementations of $\sigma$:
(a) A* on a unit grid with 8-connected moves and a Euclidean heuristic (see Sections \ref{sec:exp-time}–\ref{sec:exp-ablation}), and
(b) a Vision-Language Model (VLM)-based planner (see Section \ref{sec:exp-predictors}).
Both predictors ignore system dynamics and output paths as sequences of planar positions $(x, y)$, neglecting other components of the robot’s state. Consequently, for robot models (ii)–(iii), the generated paths may not be dynamically feasible and should be regarded as approximate, as they omit non-positional state variables (see Section \ref{sec:meth-cp}).
For model (i), A* yields feasible but not necessarily optimal paths. Moreover, A* guarantees obstacle-free paths by construction for all models. In contrast, the VLM-based predictor may generate lower-quality paths as they may collide with obstacles, across all robot models, due to their lack of correctness guarantees. 

\textbf{Conformal Prediction Setup:}
We constructed five calibration datasets, one for each distribution $\ccalD_i$, each consisting of $50$ planning problems sampled from $\ccalD_i$ and paired with their corresponding approximate optimal solutions (instead of the true ones required in~\eqref{eqn:score}) generated using RRT*. Computing the true optimal solutions is computationally intractable, as RRT* requires an infinite number of iterations to converge. Therefore, following Remark~\ref{rem:calibNCS}, we approximate the optimal solutions by running RRT* for $20,000$ iterations.

\textbf{Baselines:} We evaluate the performance of CP-RRT* against three variants of RRT* that differ only in their sampling strategies:
(a) a uniform sampling strategy,
(b) a Gaussian goal-biased strategy that samples from a Gaussian distribution---centered at the goal region and a diagonal matrix covariance matrix where all diagonal entries are equal to $10^2$---with probability $0.1$ and uniformly otherwise~\cite{lavalle1998rapidly}, and
(c) a learning-based strategy that samples from a Conditional Variational Autoencoder (CVAE) model with probability $0.5$ and uniformly otherwise~\cite{ichter2018learning}. We refer to these planners as RRT*, GB-RRT* (Goal-Biased RRT*), and CVAE-RRT*, respectively. All methods use the same parameters and share identical implementations for all other functions (e.g., steering, extending, and rewiring) for each robot model, differing only in their sampling strategy. GB-RRT* biases sampling toward the goal region, a common heuristic that can improve planning efficiency. CVAE-RRT* is a learning-based baseline that has shown strong performance but requires training a CVAE model to generate samples. %
Five separate CVAE models are trained—one for each distribution $\ccalD_i$—using environments sampled from the corresponding $\ccalD_i$ and following the authors’ released code.

\begin{figure*}[t]
 \centering
 \includegraphics[width=0.95\textwidth]{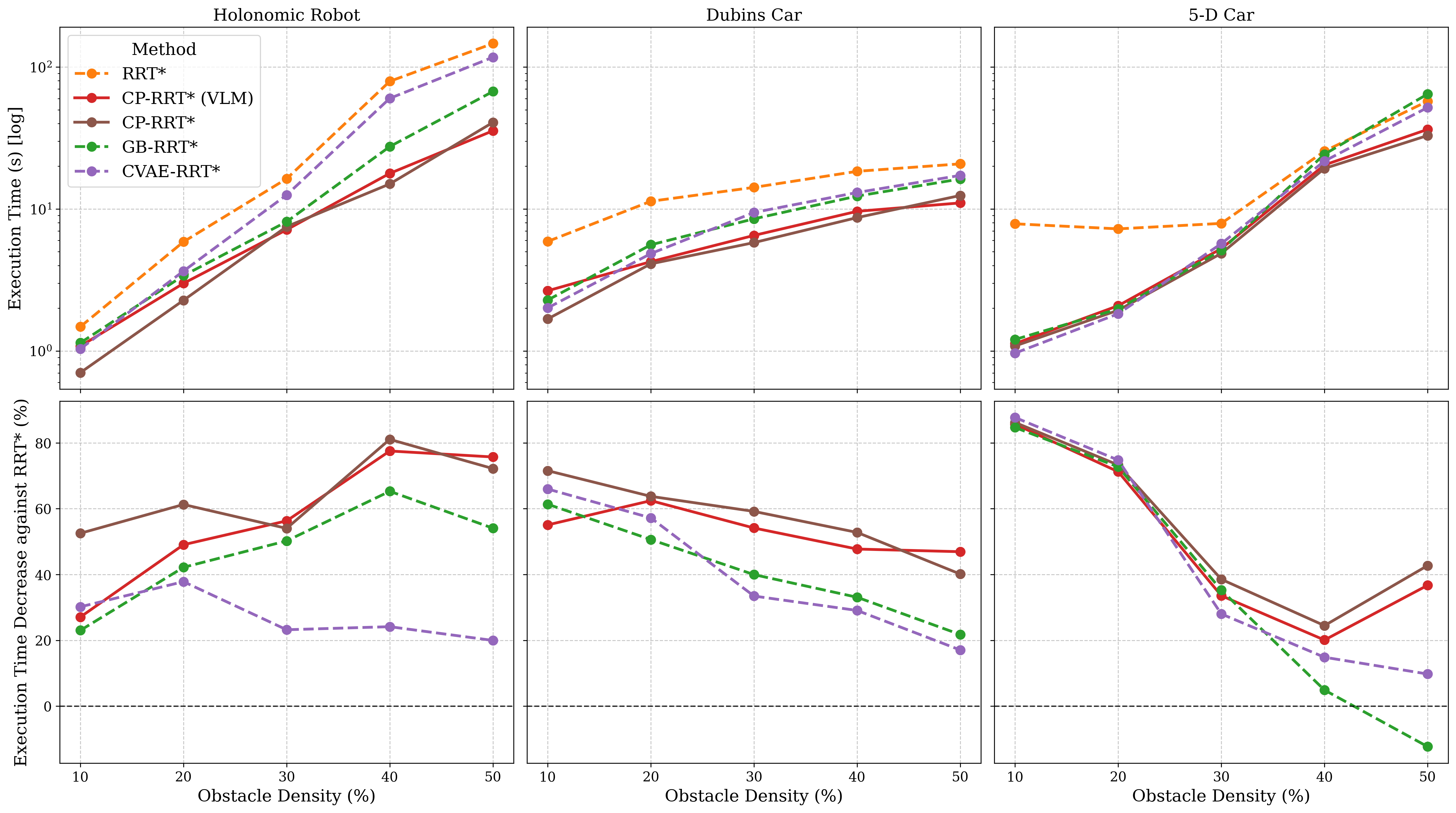}%
 \caption{Comparison of CP-RRT* against baselines in different obstacle density scenarios in terms of raw execution times (top) and the percentage of improvement in  execution time against RRT* for each set of dynamics (bottom). The first, second, and third column correspond to the holonomic robot, Dubins car, and 5-D car, respectively.}%
\label{fig:cs1-time}
\end{figure*}

\textbf{Evaluation Metrics:} Each method is executed $30$ times per planning problem, over a test set of $50$, $40$, and $20$ planning problems for robot models (i)–(iii), respectively. We consider two main evaluation metrics:
(a) the average runtime required to compute the first feasible solution, and
(b) the success rate, defined as the percentage of runs that found a feasible solution within $20,000$ iterations. We also report the average cost of the first-found paths generated by each method.
For CVAE-RRT*, the reported runtime excludes the time required to train and load the CVAE model. For CP-RRT* paired with a VLM, the reported runtime excludes the VLM inference time used to generate the initial path. This choice is deliberate because (i) VLMs are not necessarily designed for fast inference and their runtime is model-specific, and (ii) our experiments in Section~\ref{sec:exp-predictors} aim to assess how CP-RRT*’s performance changes when initialized with lower-quality paths compared to those produced by A*, rather than to evaluate the VLM’s efficiency as a path predictor. %
Our reported times for CP-RRT* paired with A* include all computational components of Alg. \ref{alg:cp-RRT*}.

\subsection{Comparative Experiments: In-Distribution Settings}
\label{sec:exp-time}
We compare CP-RRT* (paired with A*) against the baselines discussed in Section~\ref{sec:exp-setup} across environments of increasing obstacle density, generated by the distributions~$\ccalD_i$. Our method is configured with $\alpha=0.1$ and $p_{\text{bias}}=0.5$; the effect of these parameters is analyzed separately in Section~\ref{sec:exp-ablation}.
Figure~\ref{fig:cs1-time} summarizes the results, reporting the average runtimes of all methods (top row) and the percentage of improvement achieved by each non-uniform sampling strategy over uniform sampling (bottom row) for the three robot models (i)--(iii).\footnote{Absolute runtimes are not comparable across dynamics due to independent implementations. However, comparisons for a fixed robot model are valid, as all planners share identical code except for their sampling strategy.} Overall, CP-RRT* consistently outperforms all baselines, with performance gains that become more pronounced as obstacle density increases. The success rates of our method and all baselines are $100\%$ across all robot models and planning problems.

\noindent\textbf{Holonomic robot:} For holonomic dynamics, CP-RRT* reduces the average planning time by 50–80\% relative to RRT* across all densities, and by 8–46\% relative to the best-performing non-uniform baseline. \\
\noindent\textbf{Dubins car:} Similar trends are observed for the Dubins car, where CP-RRT* achieves 40–75\% lower average runtime than RRT* and remains 8–46\% faster than the strongest non-uniform baseline. The performance gap widens with increasing obstacle density.\\
\noindent\textbf{5-D car model:} For the higher-dimensional car model, CP-RRT* again provides substantial gains, reducing average runtime by 24–86\% compared to uniform RRT*. At low obstacle densities (10--20\%), performance is comparable to other non-uniform baselines, but at higher densities, CP-RRT* outperforms them by 6--37\%. Interestingly, GB-RRT* performs worse than RRT* at 50\% obstacle density, suggesting that heuristic goal biasing may require careful tuning as environment complexity and dynamics vary. Across all models and densities, all methods achieve 100\% success rates and produce paths of comparable quality, with average costs differing by less than 5\%.

\begin{figure}[t]
    \centering
    \includegraphics[width=0.7\columnwidth]{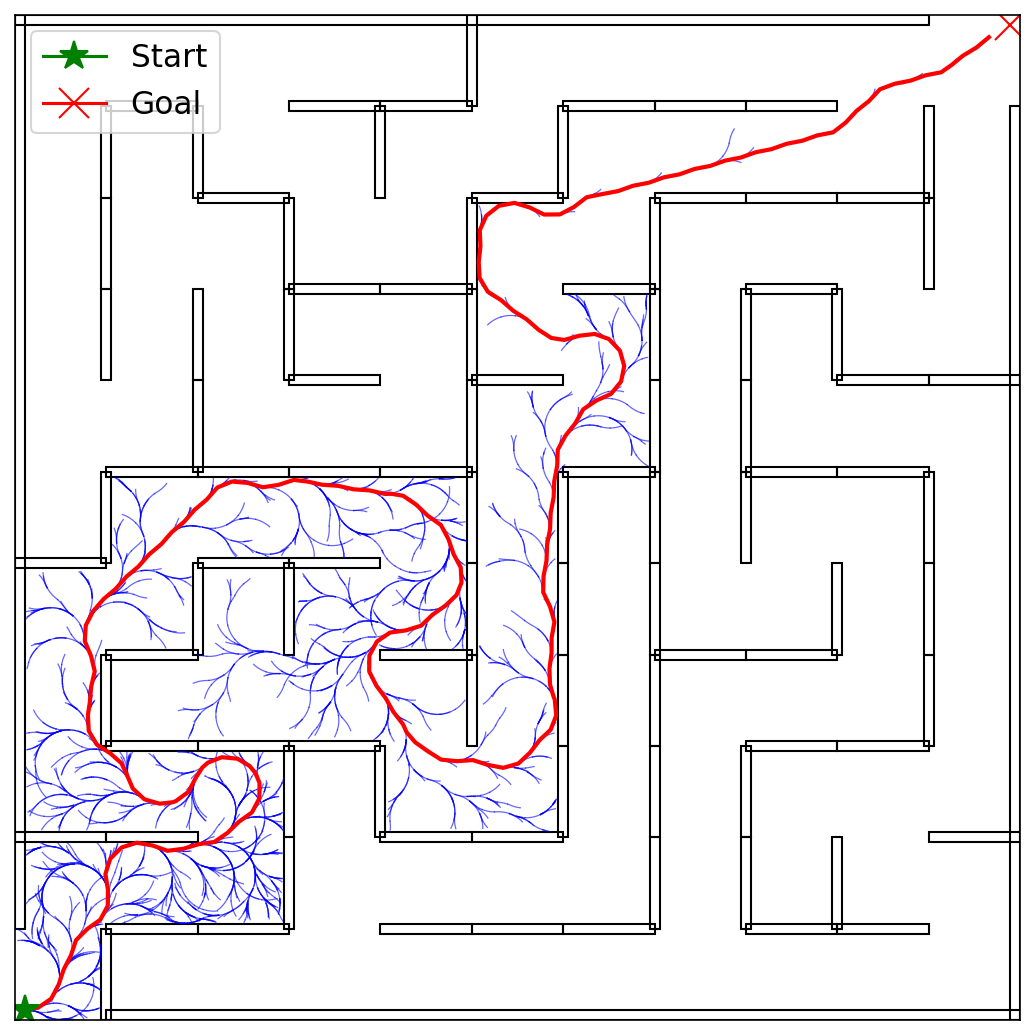}%
    \caption{Example of a maze environment along with a tree constructed by CP-RRT*.}%
    \label{fig:maze}
\end{figure}

\subsection{Comparative Experiments: Out-Of-Distribution Settings}
\label{sec:exp-generalization}

Next, we evaluate our method in $10$ randomly generated maze environments (\textcolor{black}{see Fig.~\ref{fig:maze}}) that are out-of-distribution (OOD), as they are not generated by the distributions $\ccalD_i$ and, consequently, the guarantee in~\eqref{eqn:guarantee} may not hold.
CP-RRT* is configured using the quantile computed over calibration data generated by $\ccalD_{10}$, while CVAE-RRT* uses the model trained on environments sampled from $\ccalD_{10}$. GB-RRT* remains unchanged from Section~\ref{sec:exp-time}, as it is distribution-agnostic.

\begin{table}[t]
\centering
\begin{tabular}{l cc cc}
\toprule
& \multicolumn{2}{c}{\textbf{Dubins Car}} & \multicolumn{2}{c}{\textbf{5-D Car}} \\
\cmidrule(lr){2-3} \cmidrule(lr){4-5}
\textbf{Method} & \textbf{Improv.} & \textbf{Success} & \textbf{Improv.} & \textbf{Success} \\
\midrule
RRT$^*$        & --            & 98\%  & --            & 45\% \\
GB-RRT$^*$     & 1.8\%           & 100\% & -15.36\%           & 40\% \\
CVAE-RRT$^*$   & 4.2\%           & 100\% & 2.6\%          & 48\% \\
CP-RRT$^*$     & \textbf{24.1}\% & 100\% & \textbf{14.82}\% & \textbf{65}\% \\
\bottomrule
\end{tabular}
\footnotesize
\caption{Comparative experiments in OOD settings.
}
\label{tab:generalization}
\end{table}
We compare CP-RRT* (paired with A*) against CVAE-RRT* and GB-RRT* for the Dubins and 5-D car models. Each planner is executed $20$ times per test environment with a maximum of $20,000$ iterations, and we report the percentage improvement over RRT* and the success rate of finding a solution. The results are summarized in Table~\ref{tab:generalization}.
In these OOD settings, our method demonstrates significant performance gains compared to the baselines. For the Dubins car, CP-RRT* reduces planning time by 24.1\% relative to RRT* while maintaining a 100\% success rate. GB-RRT* and CVAE-RRT* offer only marginal improvements over RRT*, with maximum gains below 5\%. The advantage is even more pronounced for the 5-D car: CP-RRT* achieves a 14.8\% speedup over RRT* and the highest success rate (65\%), whereas GB-RRT* and CVAE-RRT* perform comparably to RRT*, with success rates not exceeding 48\%.

\begin{figure*}[t]
 \centering
 \includegraphics[width=0.95\textwidth]{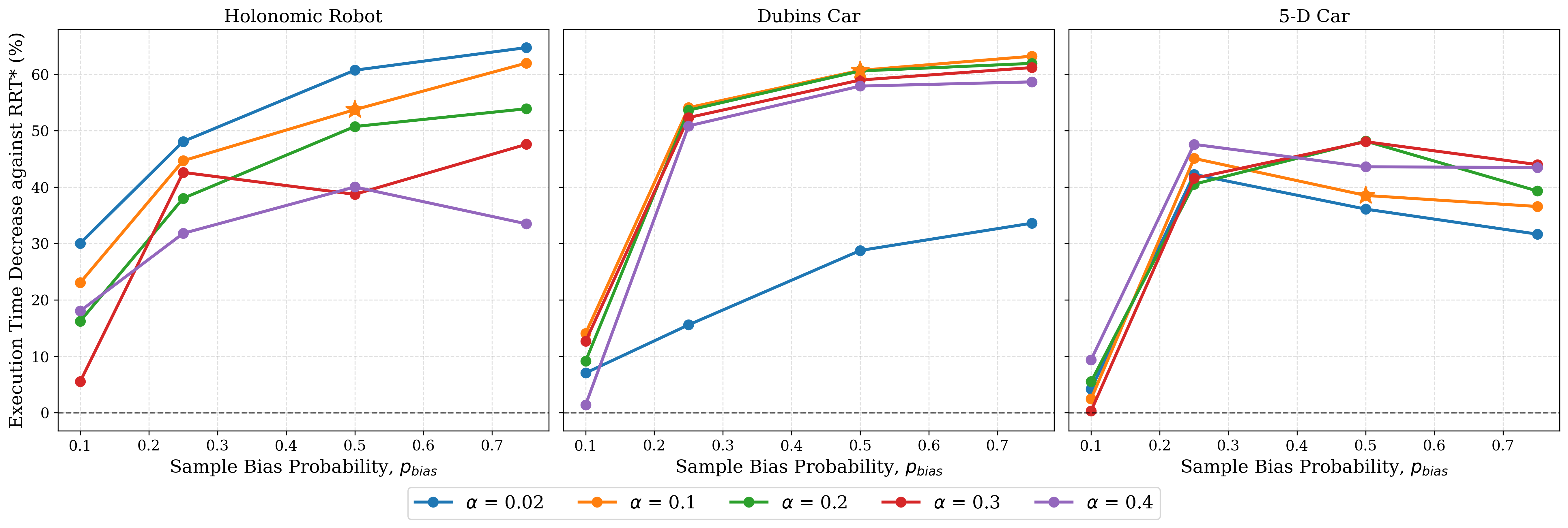}%
 \caption{Evaluation of CP-RRT* across different values of $p_{\text{bias}}$ and $\alpha$ in terms of execution time for the 30\% obstacle-density case. The star marker denotes the parameter setting used in Section \ref{sec:exp-time}/Fig. \ref{fig:cs1-time}.}%
 \label{fig:cs1-params}
\end{figure*}
\subsection{Effect of Planning Parameters}
\label{sec:exp-ablation}

We next analyze how the parameters $\alpha$ and $p_{\text{bias}}$ influence the performance gains of CP-RRT* (paired with A*) relative to RRT* in environments generated by $\ccalD_{30}$ for each  model (i)–(iii). The results are summarized in Fig.~\ref{fig:cs1-params} for various values of $\alpha$ and $p_{\text{bias}}$.
\textcolor{black}{Across all models, we observe that for a fixed $\alpha$, decreasing $p_{\text{bias}}$ leads to smaller performance gains over RRT*. This is expected, as when $p_{\text{bias}}\to 0$, our sampling strategy gradually degenerates to uniform sampling.}

\noindent\textbf{Holonomic robot:}
For the holonomic robot, given a fixed $\alpha \in \{0.02, 0.1, 0.2, 0.3\}$, increasing $p_{\text{bias}}$ improves performance over RRT*. A similar trend holds for $\alpha = 0.4$, except that performance drops when $p_{\text{bias}} = 0.75$. We attribute this drop to the fact that $\alpha = 0.4$ produces less informative prediction sets that contain the optimal solution with only $0.6$ probability. Moreover, for a fixed $p_{\text{bias}}$, lower values of $\alpha$ generally yield better performance, since smaller $\alpha$ values correspond to more informative prediction regions that are likely to include the optimal solution with higher probability $1 - \alpha$. We note that this trend does not hold indefinitely: excessively small $\alpha$ values can produce overly large prediction sets, causing our sampling strategy to approach uniform sampling regardless of $p_{\text{bias}}$.

\noindent\textbf{Dubins car:}  Similar trends are observed for the Dubins car model. The main difference is that CP-RRT*’s performance appears less sensitive to variations in $\alpha$ for a fixed $p_{\text{bias}}$. Also, observe that $\alpha = 0.02$ produces a smaller performance improvement than other $\alpha$ values. This is explained by the reasoning presented above as arbitrarily small $\alpha$ values induce large prediction sets. The best performance improvement ($\approx 64\%$) is achieved when $\alpha = 0.1$ and $p_{\text{bias}} = 0.75$, which is comparable to the best configuration for the holonomic robot ($\alpha = 0.2$, $p_{\text{bias}} = 0.75$).

\textbf{5-D car model:} For the 5-D car model, the observed trends deviate somewhat from those seen in the lower-dimensional cases. The best performance improvement ($\approx 48\%$) occurs when $\alpha = 0.4$ and $p_{\text{bias}} = 0.25$, with comparable results for $\alpha \in \{0.2, 0.3\}$ and $p_{\text{bias}} = 0.25$. We attribute this behavior to the CP setup. As discussed in Section~\ref{sec:exp-setup}, CP provides guarantees only for the 2D positional component of the robot state by construction of the initial paths. In the 5-D model, three additional state variables are ignored when forming the prediction sets, making them less informative than in the previous two models. This likely explains why performance decreases once $p_{\text{bias}}$ exceeds $0.25$, as overly emphasizing the biased sampling toward those regions may reduce exploration in the remaining three uncalibrated dimensions. As in other cases, $\alpha=0.02$ tends to yield the smallest performance improvement. In this setting, a moderate bias ($p_{\text{bias}} = 0.25$ with $\alpha \in \{0.2, 0.3, 0.4\}$) strikes the best balance between informative 2D guidance and sufficient exploration in the remaining state dimensions. Nevertheless, even for larger $p_{\text{bias}}$ values, CP-RRT* presents substantial gains—exceeding 35\% improvement over RRT* in all cases.

\subsection{Vision-Language Models as Path Predictors}
\label{sec:exp-predictors}

Finally, we study the effect of the path predictor on the performance of our method. These comparative results are presented in Figure \ref{fig:cs1-time}. In addition to A*, we test CP-RRT* with predicted paths produced by Google Gemini Pro, a pre-trained VLM. The VLM receives a rendered map image and a textual problem description that describes the planning problem (i.e, scenario, start state, end state, and obstacle geometry with coordinates). The VLM outputs an ordered sequence of states from start to goal. 
Note that the inference time of the VLM (6.73 secs on average) is higher than that of A* (0.29 secs). However, this runtime can be reduced by using smaller, fine-tuned models. We observe that although the VLM provides lower quality paths, as discussed in Section \ref{sec:exp-setup}, it still results in substantial performance gains comparable to A*. For instance, relative to RRT*, the VLM-guided CP-RRT* planner reduces mean execution time by 25--75\% for the holonomic and Dubins models, and by 20--86\% for the 5-D car model. When compared against the strongest non-CP baseline at each obstacle density, CP-RRT* (VLM) is initially comparable or slightly slower in low-density scenarios but becomes significantly faster as environmental clutter increases, with improvements ranging from 6\% to 47\% in high-density cases. Furthermore, planning time variability is also considerably lower, with standard deviations falling by up to 85\% relative to RRT*.

\section{Conclusion}
In this paper, we presented the first certified non-uniform sampling strategy for SBMPs that biases sampling toward regions that are, with user-specified probability, guaranteed to contain the
optimal path. Our comparative evaluations showed that our proposed method results in significant performance gains compared to existing baselines in terms of runtimes to compute the first feasible solution and generalization to unseen and OOD environments.

\bibliographystyle{ieeetr}
\bibliography{IEEEabrv, IEEEexample}

\end{document}